\documentclass{article}

\usepackage{microtype}
\usepackage{graphicx}
\usepackage{subfigure}
\usepackage{booktabs} 

\usepackage{multirow} 
\usepackage{booktabs} 
\usepackage{arydshln} 
\usepackage{bbding, pifont} 
\usepackage{color,soul} 
\usepackage{amsmath}
\usepackage{amssymb}
\usepackage{mathtools}
\usepackage{amsthm}

\usepackage{hyperref}

\usepackage[accepted]{icml2024}

\usepackage{amsmath}
\usepackage{amssymb}
\usepackage{mathtools}
\usepackage{amsthm}

\usepackage[capitalize,noabbrev]{cleveref}

\theoremstyle{plain}
\newtheorem{theorem}{Theorem}[section]
\newtheorem{proposition}[theorem]{Proposition}

\theoremstyle{definition}

\theoremstyle{remark}
\newtheorem{remark}[theorem]{Remark}

\usepackage[textsize=tiny]{todonotes}

\icmltitlerunning{Harnessing Neural Unit Dynamics for Effective and Scalable Class-Incremental Learning}

\begin{document}

\twocolumn[
\icmltitle{Harnessing Neural Unit Dynamics for Effective and Scalable\\ Class-Incremental Learning}



\icmlsetsymbol{equal}{*}

\begin{icmlauthorlist}
	\icmlauthor{Depeng Li}{sch1}
	\icmlauthor{Tianqi Wang}{sch1}
	\icmlauthor{Junwei Chen}{sch1}
	\icmlauthor{Wei Dai}{sch2}
	\icmlauthor{Zhigang Zeng}{sch1}
\end{icmlauthorlist}

\icmlaffiliation{sch1}{School of Artificial Intelligence and Automation, Huazhong University of Science and Technology, Wuhan, China}
\icmlaffiliation{sch2}{School of Information and Control Engineering, China University of Mining and Technology, Xuzhou, China}

\icmlcorrespondingauthor{Zhigang Zeng}{zgzeng@hust.edu.cn}

\icmlkeywords{Machine Learning, ICML}

\vskip 0.3in
]



\printAffiliationsAndNotice{}  

\begin{abstract}
Class-incremental learning (CIL) aims to train a model to learn new classes from non-stationary data streams without forgetting old ones. In this paper, we propose a new kind of connectionist model by tailoring neural unit dynamics that adapt the behavior of neural networks for CIL. In each training session, it introduces a supervisory mechanism to guide network expansion whose growth size is compactly commensurate with the intrinsic complexity of a newly arriving task. This constructs a near-minimal network while allowing the model to expand its capacity when cannot sufficiently hold new classes. At inference time, it automatically reactivates the required neural units to retrieve knowledge and leaves the remaining inactivated to prevent interference. We name our model AutoActivator, which is effective and scalable. To gain insights into the neural unit dynamics, we theoretically analyze the model’s convergence property via a universal approximation theorem on learning sequential mappings, which is under-explored in the CIL community. Experiments show that our method achieves strong CIL performance in rehearsal-free and minimal-expansion settings with different backbones.
\end{abstract}

\section{Introduction}
Contrary to typical machine learning methods that work on independent and identically distributed data, class-incremental learning (CIL) tackles the problem of training a single model on non-stationary data distributions. In this scenario, tasks typically consist of subsets of disjoint classes that are presented sequentially, without providing task identities at inference time~\cite{wang2023beef}. However, with data of the current task accessible but none (at least the bulk) of the past, CIL is challenged by a sharp performance decline on the previously learned tasks, known as catastrophic forgetting problem~\cite{mccloskey1989catastrophic}.

Recently, CIL of neural networks has seen explosive growth in striving for less forgetting~\cite{bonicelli2022effectiveness, tong2023incremental, qiao2024class, li2024MVCNet}. Prior works fall into three main categories~\cite{masana2022class}. \textit{Rehearsal-based approaches} maintain a small portion of past samples and mix them with that of a new task at either input layer (pixel level)~\cite{NIPS2017GEM, bang2021rainbow} or hidden layer (internal representations)~\cite{van2020brain, hayes2020remind}. However, this line of work suffers from substantial performance degradation with a smaller buffer that carries inadequate task-specific knowledge and becomes infeasible when a rehearsal buffer is not allowed due to memory constraints or privacy issues. \textit{Regularization-based approaches} aim to minimize the impact of learning new tasks on the weights~\cite{PNAS2017EWC, wolczyk2022continual} or feature representations~\cite{ICML2018Online_EWC, li2024CLDNet} that are important for previous tasks. Although avoiding data storage, the involved penalty terms make a fixed-size model rather inflexible to find the optimal solutions as it retains the memory of previous classes entirely in the parameter space. \textit{Architecture-based approaches} dynamically adapt network components by expansion~\cite{CVPR2021EFT, yang2022dynamic} or mask~\cite{serra2018overcoming, ke2021achieving} operation to absorb knowledge for novel classes. Nevertheless, network expansion usually renders the model size grow quickly as each session proceeds, which should be counted into the memory budget for a fair comparison~\cite{zhou2022model}. 

\begin{figure*}[ht]
	\begin{center}		\centerline{\includegraphics[width=1.9\columnwidth]{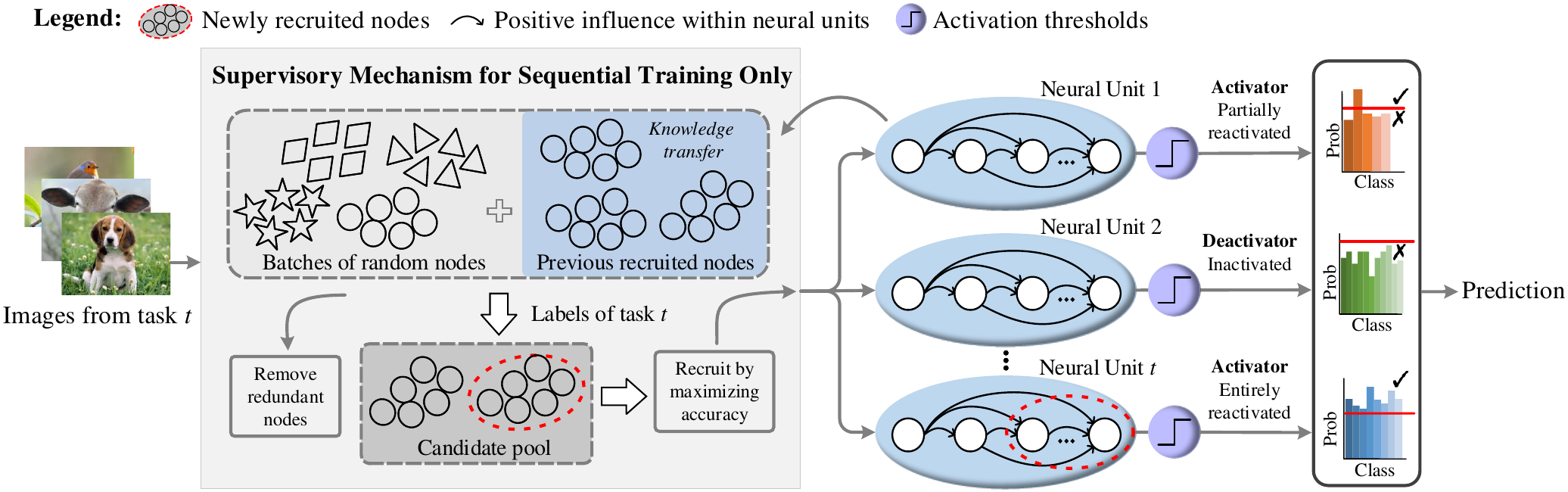}}
		\caption{Overview of AutoActivator. It first generates several batches of random nodes, denoted by different shapes; Then, together with existing ones for knowledge transfer, it parsimoniously recruits new nodes meeting the supervisory mechanism (e.g., in red circles) to a scalable neural unit, where those joined ones are positively influenced by each other as marked by black arrows. In AutoActivator, the former layers are built under the guidance of supervisory mechanism (Section~\ref{Sec_Supervisory_Mechanisms}) while the final classifier layer is step-wise updated by close-formed solutions (Section~\ref{Sec_Reactivation}). The activation thresholds render a neural unit partially/entirely active or inactivated for prediction.}
		\label{Overview}
	\end{center}
	\vskip -0.2in
\end{figure*}


Yet, in the brain---which clearly has implemented an efficient and scalable function for incremental learning---the \textit{reactivation of neuronal activity patterns} that represent previous experiences is believed to be important for stabilizing new memories~\cite{rasch2007maintaining, o2010play}. Motivated by the principle of learning and memory in cognitive neuroscience, this paper proposes a new kind of connectionist model that automatically reactivates the involved collections of nodes (dubbed neural units)\footnote{Herein the node means a neuron and incoming weights and bias associated with it and neural unit is collection of such nodes.} as activators to retrieve knowledge while leaving the remainder as deactivators to avoid crosstalk, hence the name AutoActivator. This is achieved by harnessing neural unit dynamics that adapt the behavior of neural networks for CIL. As shown in Figure~\ref{Overview}, AutoActivator pioneers a novel CIL paradigm that runs through the training and test phase: 

As opposed to the over-parameterized or expanding-and-pruning implementations, one can start with modeling the neural unit from scratch and dynamically grow the network as actually needed by a given task. Specifically, in each training session, we first randomly allocate several batches of random nodes together with the counterparts from earlier sessions and introduce a supervisory mechanism to remove the redundant nodes for the current session. Those meeting the supervisory mechanism are temporarily on standby in the candidate pool. Then, only one batch that contributes to causing a maximum reduction in training errors is added to the corresponding neural unit, i.e., the recruited nodes are positively influenced by the existing ones such that each plays an irreplaceable role. This way constructs a minimalist network for sequential tasks and thus suffices to train more incoming tasks. Meanwhile, we parameterize each neural unit with an activation threshold measured by the predicted probability during training, which guards the decision boundary of learned classes against distortion. At inference time, given batches of test instances from a certain task or class trained, the reactivation of neural units is performed partially or entirely without knowing the task identities. 

\textbf{Our main contributions are threefold:}

\begin{itemize}
	\item We design neural unit dynamics that govern the behavior of neural networks, including the rules of node generation/connection, activation threshold, and update, as well as interactions responding to non-stationary data streams. The model’s convergence property on learning sequential mappings is theoretically guaranteed.
	
	\item AutoActivator is an efficient and scalable CIL method, characterized by parsimoniously constructing a CIL model whose complexity is commensurate with the intrinsic complexity of each learning task. The method is inherently immune to catastrophic forgetting as neural units reactivated partially or entirely do not infringe upon others, and has strong task-order robustness. 
	
	\item Experiments on multiple benchmark datasets consistently demonstrate that our method provides competitive CIL performance, with absolute superiority of rehearsal-free and minimal-expansion desiderata.
\end{itemize}

\section{Related Work}
\textbf{Class-Incremental Learning.}
We discuss a selection of representative CIL approaches and how they relate to our work. Rehearsal-based approaches explicitly preserve data from previously learned tasks to retrain with the current task. Aided with rehearsal buffers, IL2M~\cite{ICCV2019IL2M} rectifies the network predictions, RM~\cite{bang2021rainbow} focuses on the classification uncertainty to select hard samples, and i-CTRL~\cite{tong2023incremental} is founded on structured representations for rehearsal. Our method does not buffer past data for the whole CIL process and thus eliminates shortcomings such as scalability and privacy issues. Regularization-based approaches involve penalty terms to vary the plasticity of parameters. EWC~\cite{PNAS2017EWC} is the pioneer of this branch, followed by SI~\cite{ICML2017SI}, and MAS~\cite{ECCV2018MAS}. Their network parameter is associated with the weight importance computed by different strategies. Nevertheless, the challenge is to correctly assign credit to the network weights when the number of tasks is large. Our method keeps the neural unit weights intact during sequential training, i.e., the newly recruited nodes without infringing upon others. Architecture-based approaches isolate existing model parameters or attach additional parameters as each session proceeds. Methods such as DER~\cite{yan2021dynamically}, FOSTER~\cite{wang2022foster}, and DNE~\cite{hu2023dense} acquire sufficient learning capacity by adding a sub-network per task. However, the increased capacity for future tasks must be meticulously balanced with the number of parameters added, particularly considering that the number of tasks the model needs to learn is often unknown in advance. Our method differs in that (i) the expansion quota is commensurate with the intrinsic complexity of each task; (ii) it's more effective to start with a small/compact branch instead of scaling arbitrarily and pruning; and (iii) empirically, our final memory budget is comparable or even superior to the non-growing networks in regularization-based methods.

\textbf{Neural Network Dynamics.}
The learning dynamics of a connectionist model, such as an artificial neural network (ANN), refers to how the network's internal state evolves over time in response to inputs~\cite{vahedian2021convolutional, marton2022linking}. This can involve different levels, including the connections of the neurons, the activation patterns, the flow of information through the network's layers, and the overall convergence and learning behavior~\cite{vyas2020computation}. 
For example, training a feed-forward neural network with random hidden nodes has been explored in the single-task learning~\cite{pao1992functional, li2017insights}. The general idea is to randomly generate hidden nodes (weights and biases), and only output weights need to be tuned in either a deterministic or nondeterministic manner. Such a randomized learning dynamic has demonstrated great potential in developing fast learner models and easy-implementation learning algorithms, such as convolutional/graph randomized networks \cite{zhang2016visual, huang2022graph}. By bridging this intriguing learning dynamic to the architectural update paradigm, our work attempts to build an efficient and scalable network for CIL.

\section{Class-Incremental Learning Setup}
\label{CIL_setup}
\textbf{Notations.} Denote ${\Gamma}=\{{g}_1, {g}_2, \dots\}$ as a set of bounded nonconstant piecewise continuous functions, span(${\Gamma}$) as a function space spanned by ${\Gamma}$, and $L_2({D}_t)$ as the space of all Lebesgue measurable functions ${f} = [{f}_1, {f}_2, \dots, {f}_{C_t}]: \mathbb{R}^{M_t}\rightarrow \mathbb{R}^{C_t}$ defined on ${D}_t$. Hence, $L_2$ norm is defined as
\begin{equation}
	\Vert{f}\Vert = \bigg(\sum_{c=1}^{C_t}\int_{{D}_t}|{f}_c(x)|^2 dx\bigg)^{\frac{1}{2}}<\infty
\end{equation}
\noindent The inner product of ${\vartheta}=[{\vartheta}_1, {\vartheta}_2, \dots, {\vartheta}_{C_t}]: \mathbb{R}^{M_t}\rightarrow \mathbb{R}^{C_t}$ and ${f}$ is further formulated as
\begin{equation}\label{product}
	\langle{f},{\vartheta}\rangle = \sum_{c=1}^{C_t}\langle{f}_c,{\vartheta}_c\rangle = \sum_{c=1}^{C_t}\int_{{D}_t}{f}_c(x){\vartheta}_c(x)dx
\end{equation}
We now define CIL formally. A model sequentially learns from the supervised learning datasets ${D}_t=\{({X}_t,{Y}_t)|{X}_t\in \mathbb{R}^{N_t\times M_t}, {Y}_t\in \mathbb{R}^{N_t\times C_t}\}$ of task $t$ ($t=1,2,\dots,T$), where ${X}_t$ is the data, ${Y}_t$ is the label, $N_t$ is the number of samples, $M_t$ and $C_t$ are the dimensions, respectively. The model $\mathcal{M}({X}_{t-1};{\theta}_{t-1})$~$(t\geq2)$ trained on previous task(s) is parameterized by its connection weights ${\theta}_{t-1}$. The objective is to train an updated model $\mathcal{M}({X}_t;{\theta}_t)$ that accommodates the newly emerging $C_t$ classes, during which the data of previous tasks is inaccessible. When fed test instances from any of tasks 1 to $T$, the model $\mathcal{M}({X}_T;{\theta}_T)$ can make predictions without task descriptors/identifiers.

\textbf{Fair Comparisons.}
Since different CIL methods have very different requirements in data, networks, and computation, it is intractable to compare all under the same experimental conditions. Following the suggestion in~\cite{zhou2022model}, we holistically evaluate different methods by considering both accuracy and memory cost for a fair comparison. We align the memory cost of model size and exemplar buffer (if any) by switching them to a 32-bit floating number which we refer to as~\textit{memory budget}. Also, with the increasing prominence of foundation models, pre-trained models equipped with informative representations have become available for various downstream
requirements~\cite{mcdonnell2023ranpac, mehta2023empirical}. Following the settings in prior work~\cite{cha2021co2l, rios2022incdfm, bonicelli2022effectiveness, tang2023prompt}, we optionally inject AutoActivator into some advanced pre-trained backbones such as ResNet (used in PCL~\cite{AAAI2021PCL}, OWM~\cite{NMI2019OWM}, etc.) and ViT (used in DualPrompt~\cite{wang2022dualprompt}, CODA-Prompt~\cite{smith2023coda},  etc.). Such a setting accommodates real-world scenarios where pre-training is usually involved as a base session~\cite{bonicelli2022effectiveness}. We conduct experiments across multiple datasets with or without starting with the same pre-training.

\section{Methodology}
\label{Methodology}
Our AutoActivator is a new kind of connectionist model (see Figure~\ref{Overview}), with tailored rules of node generation/connection, activation threshold, and update, as well as interactions responding to sequential tasks. To obtain a good grasp of the CIL model, we first provide the theoretical guide to node expansion in Section~\ref{Sec_Supervisory_Mechanisms}. We then perform the reactivation of involved neural units for learning-without-forgetting decision-making in Section~\ref{Sec_Reactivation}. 

\subsection{Modeling Neural Units via Supervisory Mechanism}
\label{Sec_Supervisory_Mechanisms}
Instead of empirically over-parameterized or expanding-and-pruning implementations, we seek a brand-new solution with theoretical support during sequential training. With this consideration, we start with modeling a neural unit from scratch and then grow additional nodes as the given problem requires. In this way, the expansion quota is compactly commensurate with the intrinsic complexity of each task, and thus constructs near-minimal neural network architectures for the CIL process. However, two main problems are: (1) How to generate and connect new nodes to a scalable neural unit? (2) How to update and then reactivate the neural units recruited without recourse to task identities?

To answer the first question, we draw inspiration from recent advances in network randomization~\cite{huang2022graph, ramanujan2020s, wang2017stochastic}, a randomized learning technique for developing fast learner models and easy-implementation learning algorithms. A common and basic idea behind this technique is to randomly generate hidden nodes (weights and biases), and only output weights need to be tuned~\cite{pao1992functional, zhang2016visual, li2017insights, zhang2022all}. In this way, one can add new nodes with random weights to different network layers progressively. Specifically, given a target function $f: \mathbb{R}^{M_t}\rightarrow \mathbb{R}^{C_t}$ of task $t$ $(t=1,2,\dots)$ defined on datasets ${D}_t=\{({X}_t,{Y}_t)\}$, we suppose a neural unit has been added $L-1$ nodes one after another directly connected to its readout layer. During each training session $t$, the output function of the existing network is given by $ f_{L-1}( X_t) = \sum_{j=1}^{L-1}\beta_j(t) g_j( X_t  w_j + b_j)$ $(L=1,2,\dots, f_0 =  0)$ and the current residual error for training is denoted as $ e_{L-1}(t) =  f -  f_{L-1}$, where $w_j$ and $b_j$ are the hidden parameters, and $\beta_j$ is the output weights. Then, the following Proposition~\ref{Proposition_1} provides a solution to connect the newly generated node $ g_L$ ($ w_L$ and $b_L$) to the existing network $ f_{L-1}$ when $C_t = 1$.

\begin{proposition}
	\cite{kwok1997objective}
	\label{Proposition_1}
	Let $\Gamma$ be a set of basis functions $ g$. For a fixed $ g \in \Gamma$ $(\Vert  g \Vert \neq 0)$, the expression $\Vert  f - ( f_{L-1} +  \beta_L  g_L)\Vert$ achieve its minimum iff
	\begin{equation}\label{Bete_L}
		\beta_{L} = \langle e_{L-1},  g_L\rangle/{\Vert  g_L \Vert^2}
	\end{equation}
\end{proposition}

Proposition~\ref{Proposition_1} lays the groundwork for how to connect a new node to the existing network but is \textit{limited to the regression problem in single-task learning}. Meanwhile, constructing such a new node by Eq. (\ref{Bete_L}) is less practical in the sense that the reduction of residual error per node expansion will be close to zero~\cite{igelnik1995stochastic, wang2017stochastic}, failing to guarantee preferable learning performance with considerable confidence and convergence rate. The main cause is that fixed random weights, which only perform the forward pass
without the backward pass, are prone to incur numerous redundant nodes. Therefore, alternatively, random node generation should be “supervisory” by exerting additional conditions in the forward pass. To this end, we introduce a supervisory mechanism to guide the node generation in modeling neural units for CIL. 

In AutoActivator, the former layers are built under the guidance of supervisory mechanism where each layer is made up of a certain number of scalable neural units, while the final classifier layer is step-wise updated by close-formed solutions. The method randomly allocates a batch of nodes in each training session. The batch that causes the highest reduction in training error is added to the neural unit of the current task. For simplicity, Theorem \ref{Theorem_2} formulates neural units over a two-layer AutoActivator, which can be stacked or injected into existing backbones.

\begin{theorem} (Universal Approximation Theorem for Convergence Property) 
	\label{Theorem_2}
	Suppose that span($\Gamma$) is dense in ${L_2}$ space and $\forall g \in \Gamma $, $G$ is a collection of some $ g$ with nonlinear activation. Given $0 < r(t) < 1$ and a non-negative real number sequence $\{\mu_L(t)\}$ w.r.t. task $t$ ($t = 1,2,\dots; c = 1,\dots,C_t$), with $\lim_{L \to +\infty }\mu_L(t) = 0$ and $\mu_L(t) \le (1 - r(t))$. For $L \footnote{Without loss of generality, we refer to $L = L(t)$ for simplicity.} = 1,2,\dots,$ and step size $l \in \mathbb{N}^+$, denoted by
	\begin{equation} \label{delta_2}
		\delta_L^\star(t)\! =\! \sum_{c=1}^{C_t}\delta_{L,c}^\star(t), \delta_{L,c}^\star(t) \! =\! (1 - r(t) - \mu_L(t))\Vert e_{L - l,c}^ \star(t)\Vert^2
	\end{equation}
	If a batch of new nodes $ G_l( X_t  W_l +  B_l)$\footnote{Note that $ G_l =  g_l$ in the special case of $l=1$.} are randomly generated to satisfy the following inequalities:
	\begin{equation} \label{Inequalities_2}
		\langle e_{L-l,c}^\star(t),  G_l \beta_{l,c}(t) \rangle \geq \delta_{L,c}^\star(t),~ c=1,2,\dots,C_t 
	\end{equation}
	\noindent and connected to the existing neural unit through the output weights in the following least-squares sense:
	\begin{equation} \label{beta_LSM_2}
		[{\beta}_1^\star(t), \dots, {\beta}_L^\star(t)] = \arg\min_{{\beta}(t)} \Vert  f - \sum_{j=1}^{L}\beta_j(t) g_j \Vert
	\end{equation}
	Then, we have $\lim_{L \to  +\infty} \Vert  f -  f_L^\star\Vert = 0$, where $ f_L^\star =  f_{L-l}^\star + \beta_l^\star(t) G_l$ and $ G_l=[ g_{L-l+1},\dots, g_L]$.
\end{theorem}

\begin{proof} 
	See Appendix \ref{Proof_Theorem_1}. 
\end{proof}

Based on Theorem \ref{Theorem_2}, we can formulate the indicator $\xi_{L}(t) = \sum_{c=1}^{C_t}\xi_{L,c}(t)$ of supervisory mechanisms for guiding node expansion by a designated step size $l$, among which
\begin{equation} \label{xi_2}
	\begin{split}
		\xi_{L,c}(t) = & \langle e_{L-l,c}^\star(t),  G_l \beta_{l,c}(t) \rangle \\
		&- (1 - r(t) - \mu_L(t)) \langle e_{L-l,c}^\star(t),  e_{L-l,c}^\star(t) \rangle > 0
	\end{split}
\end{equation}
\noindent Intuitively, $0<r(t)<1$ matters in the residual error decreasing speed of task $t$, which resembles the learning rate of gradient descent but differs in its explicit scope; $\mu_L(t)$ can be seen as the balance coefficient for ensuring the convergence of Theorem~\ref{Theorem_2}, which can be found in its proof. 

\begin{remark}\label{Remark}
	The theoretical analysis redefines Eq. (\ref{Bete_L}) \textit{in the context of classification problem for learning sequential tasks}. Modeling neural units with supervisory mechanisms can easily add additional capacity. For one thing, it guarantees the convergence property of a model on a sequence of tasks. In particular, we generate several candidate nodes satisfying supervisory mechanisms in a one-by-one or batch-by-batch manner simultaneously and only recruit ones that cause a maximum reduction in residual errors. This is beneficial to accelerate the convergence rate. For another, during sequential training, the required expansion quota potentially matches the inherent difficulty of each newly arriving task. This not only constructs a near-minimal network architecture for CIL but also suffices to train more incoming tasks. 
\end{remark}

\subsection{Reactivating Updated Neural Units} 
\label{Sec_Reactivation}
This section addresses the aforementioned second question, i.e., update the connection weights between neural units and the readout layer (termed output weights) and then reactivate the involved neural units as activators for decision-making. Unlike neural networks with an empirically fixed topology that is unitedly trained, the updates of output weights under node expansion must be repeated every time nodes are added. Hence, computational efficiency becomes the main bottleneck when leveraging the commonly used back-propagation algorithm.

\textbf{Update.} Here we seek an alternative solution under the premise of supervisory mechanisms, which ensures that the nodes already existing in each neural unit are indispensable in learning a given task. That is the pseudoinverse of a partitioned matrix described by some earlier studies~\cite{leonides2012control, ben2003generalized, chen1999rapid}. A dynamic stepwise updating algorithm is then used to update the output weights instantly. It applies to the node expansion process in Theorem \ref{Theorem_2}, as illustrated in the following.

Assume that there has been a matrix $ A_L$ and we expand it by adding additional $l$ ($l=1,2,\dots$) node(s). Let $ G_l$ be the resulting matrix. For task $t$, outputs of the neural unit are given by $\hat{ Y}_t= A_L  W_L^\star(t)$, where $ W_L^\star(t) = [{\beta}_1^\star(t), \dots, {\beta}_L^\star(t)]$ as presented in Eq. (\ref{beta_LSM_2}). Denote by $ A_{L+l} = [ A_L,  G_l]$, we have
\begin{equation} \label{W_L+1}
	W_{L+l}^\star(t) \triangleq ( A_{L+1})^\dag  Y_t = 
	\left[\begin{array}{c}
		W_L^\star(t) - {D}{B}^\mathrm{T}{Y_t}\\
		{B}^\mathrm{T}{Y_t}
	\end{array}\right]
\end{equation}
where the pseudoinverse of $ A_{L+l}$ is
\begin{equation}\label{A_L+l_p}
	( A_{L+1})^\dag =
	\left[\begin{array}{c}
		( A_L)^\dag - {D}{B}^\mathrm{T}\\
		{B}^\mathrm{T}
	\end{array}\right]
\end{equation}
\begin{equation}\label{B.T}
	{B}^\mathrm{T}=
	\left\{ \begin{array}{ll}
		{C}^\dag & \text{if} \ {C}\neq  0\\		({D}^\mathrm{T}{D} + {I})^{-1}{D}^\mathrm{T}({A}_L)^+& \text{if} \ {C} =  0
	\end{array} \right.	
\end{equation}

\noindent where ${C}=  G_l - {A}_L{D}$ and ${D}=({A}_L)^\dag  G_l$.

This way, updating output weights has three strengths. First, the update process can be easily finished \textit{without complete retraining}, i.e., a new pseudoinverse matrix is obtained through a simple calculation of the existing ones. Second, it updates the output weights progressively when adding nodes to each neural unit. In each step, the output weights are theoretically the optimal solution in the least-square sense. Third, only output weights need to be optimized every time, as explicitly expressed in close-formed solutions. Therefore, it avoids the back-propagation of error signals.

\textbf{Reactivation.}
In the CIL scenario, task identities are typically absent at the inference time. Hence, another key technical point of AutoActivator is to partially or entirely reactivate the required neural units for task-agnostic decision-making. Herein an \textit{activation threshold} is calculated during training to be used for activating some of the neural units at the test time. Suppose $T$ tasks with $\sum_{t=1}^{T}C_t$ classes have been sequentially trained, yielding $T$ neural units. The following explains how our approach automatically performs the reactivation when fed test instances from tasks 1 to $T$.

We first retrospect the training phase to passingly introduce the activation threshold of each neural unit, averaged by the highest predicted probability belonging to class $c$ ($c=1,2,\dots, C_t$):  	
\begin{equation}\label{threshold_t}
	threshold(t) = mean.\max\!.(softmax(\hat{ Y}_t)-\alpha(t))
\end{equation}
where $\alpha(t) = 1/C_t$ is the lowest probability triggering a decision. Hence, a neural unit is finally parameterized by its supervisory mechanism-based random weights (and biases), dynamically stepwise updated output weights, and the activation thresholds. Then, for test instances of any classes seen so far, the objective is to minimize the distance: 
\begin{equation}\label{threshold_dis}
	c \leftarrow \arg\min_t |threshold(t) - \tilde{threshold}(t)|
\end{equation}
where “$\leftarrow$” means returning index, $c$ is the predicted class, and $\tilde{threshold}(t)$ is computed in Eq. (\ref{threshold_t}) but with test instances. Indeed, Eq. (\ref{threshold_t}) rectifies the inter-class confusion occurring in the similar $softmax$ outputs given test instances from different tasks; Eq. (\ref{threshold_dis}) further brings these results to a comparable level and returns the predicted class by tracking the minimal distance in terms of the activation thresholds. From the perspective of methodology, it works on test images arriving in batches or instances, among which feeding batches of test instances from a certain task or class trained facilitates the reactivation progress. Therefore, our AutoActivator can selectively reactivate the involved neural units to retrieve knowledge and leave the remaining inactivated to prevent interference. We summarize its CIL procedure 
in Algorithm~\ref{Alg_AutoActivator} in Appendix \ref{alg_1}.

\begin{table*}[t]
	\caption{Comparison results on the split MNIST and FashionMNIST datasets with MLP. We report the average classification accuracy (ACA), backward transfer (BWT) across five random task-order runs, and the aligned memory budget (MB) = model size (\#model, MB) + exemplar buffer (\#exemplar, MB), where 0 means no rehearsal. No pre-training is used for AutoActivator and all the compared methods.}
	\label{Table_MNIST_FMNIST}
	\begin{center}
		\begin{small}
			\begin{sc}
				\begin{tabular}{lcccccc}
					\toprule
					\multicolumn{1}{l}{\textbf{Method}} & 
					\multicolumn{2}{c}{\textbf{Memory budget}} & \multicolumn{2}{c}{\textbf{MNIST-10/5}} & \multicolumn{2}{c}{\textbf{FashionMNIST-10/5}} \\ 
					& \#model $\downarrow$ & \#exemplar $\downarrow$ & ACA(\%) $\uparrow$ & BWT $\uparrow$ & ACA(\%) $\uparrow$ & BWT $\uparrow$ \\ \midrule 
					SGD (lower bound) &1.82 &-  &$\sim$19.91  &-  &$\sim$19.81  &-   \\
					MTL (uper bound)&1.82 &-  &$\sim$98.56  &-  &$\sim$96.61 &-  \\ \midrule
					MAS	&5.47 &0 &44.61\scriptsize{$\pm$6.62} &$-$0.06\scriptsize{$\pm$0.06} &34.91\scriptsize{$\pm$5.47} &$-$0.61\scriptsize{$\pm$0.09} \\
					OLEWC&5.47 &0 &57.38\scriptsize{$\pm$4.04} &$-$0.38\scriptsize{$\pm$0.06}   &54.09\scriptsize{$\pm$4.03} &$-$0.40\scriptsize{$\pm$0.08}  \\
					SI &5.47 &0 &69.44\scriptsize{$\pm$4.37} &$-$0.04\scriptsize{$\pm$0.09} &52.11\scriptsize{$\pm$2.22} &$-$0.49\scriptsize{$\pm$0.06} \\
					EFT &16.54  &0&82.53\scriptsize{$\pm$1.15} &$-$0.09\scriptsize{$\pm$0.07} &74.79\scriptsize{$\pm$1.23} &$-$0.13\scriptsize{$\pm$0.05} \\
					PCL &3.01 &0 &94.14\scriptsize{$\pm$0.67} &$-$0.03\scriptsize{$\pm$0.03} &83.27\scriptsize{$\pm$0.81}  &$-$0.12\scriptsize{$\pm$0.01} \\
					AOP	&5.39 &0&94.43\scriptsize{$\pm$0.21} &$-$0.05\scriptsize{$\pm$0.02} &82.97\scriptsize{$\pm$0.95} &$-$0.14\scriptsize{$\pm$0.04} \\
					CRNet &\textbf{1.81} &0 &94.49\scriptsize{$\pm$0.32} &$-$0.02\scriptsize{$\pm$0.02} &86.01\scriptsize{$\pm$0.74} &$-$0.09\scriptsize{$\pm$0.01} \\
					FS-DGPM	&1.82 &5.98 &89.12\scriptsize{$\pm$1.14} &$-$0.08\scriptsize{$\pm$0.01} &80.89\scriptsize{$\pm$0.74} &$-$0.12\scriptsize{$\pm$0.02} \\
					NISPA &1.82 &5.98 &91.07\scriptsize{$\pm$0.86} &$-$0.04\scriptsize{$\pm$0.00} &80.93\scriptsize{$\pm$0.59} &$-$0.15\scriptsize{$\pm$0.02} \\
					BiC	&1.82 &5.98 &93.93\scriptsize{$\pm$0.58} &$-$0.04\scriptsize{$\pm$0.01} &82.36\scriptsize{$\pm$0.72} &$-$0.11\scriptsize{$\pm$0.03} \\
					ARI	&3.65 &5.98 &93.60\scriptsize{$\pm$0.57} &$-$0.04\scriptsize{$\pm$0.00} &82.89\scriptsize{$\pm$0.83} &$-$0.10\scriptsize{$\pm$0.01} \\ 
					Co$^2$L	&1.82 &5.98 &93.78\scriptsize{$\pm$0.24} &$-$0.04\scriptsize{$\pm$0.00} &80.93\scriptsize{$\pm$0.59} &$-$0.15\scriptsize{$\pm$0.02} \\ 
					RPS-Net &14.60 &5.98 &94.53\scriptsize{$\pm$1.92} &$-$0.02\scriptsize{$\pm$0.01} & 84.18\scriptsize{$\pm$1.60} &\textbf{$-$0.03}\scriptsize{$\pm$0.01}\\
					LOGD &1.82 &5.98 &94.87\scriptsize{$\pm$0.59} &$-$0.04\scriptsize{$\pm$0.01} &84.39\scriptsize{$\pm$0.47}  &$-$0.09\scriptsize{$\pm$0.03} \\
					IL2M &1.82 &5.98 &95.51\scriptsize{$\pm$0.42} &$-$0.04\scriptsize{$\pm$0.01} &82.38\scriptsize{$\pm$2.04} &$-$0.15\scriptsize{$\pm$0.03} \\ 
					AutoActivator (Ours) &2.04 &\textbf{0}  &\textbf{97.32}\scriptsize{$\pm$0.03} & \multicolumn{1}{r}{\textbf{0.00}\scriptsize{$\pm$0.00}}  &\textbf{88.46}\scriptsize{$\pm$0.06} & \multicolumn{1}{r}{-0.09\scriptsize{$\pm$0.08}} \\ \bottomrule
				\end{tabular}
			\end{sc}
		\end{small}
	\end{center}
	\vskip -0.12in
\end{table*}

\section{Experiment}
\subsection{Experiment Setting}

\textbf{Datasets.} 
We experiment on multiple datasets commonly used for CIL. \textbf{Small Scale:} Both MNIST~\cite{lecun1998gradient} and FashionMNIST~\cite{FashionMNIST} are respectively split into 5 disjoint tasks with 2 classes per task. The evaluation starts with the toy examples since AutoActivator belongs to the theoretical and embryonic approach. \textbf{Medium Scale:} CIFAR-100~\cite{CIFAR-100} is divided into 10 (25) tasks with each task containing 10 (4) disjoint classes. \textbf{Large Scale:} ImageNet-R~\cite{hendrycks2021many} has 200 classes with 24,000 samples for training and 6,000 for testing. It is split into 10 tasks with 20 classes in each task. The substantial intra-class variability renders it more akin to intricate real-world problems. We provide details of the data splits used in Appendix~\ref{DS_A}.

\textbf{Baselines.} 
We extensively compare our method with (i) representative and the latest CIL baselines: 
\textbf{IL2M}~\cite{ICCV2019IL2M}, \textbf{BiC}~\cite{wu2019large}, \textbf{LOGD}~\cite{CVPR2021LOGD}, \textbf{FS-DGPM}~\cite{deng2021flattening}, \textbf{Co$^2$L}~\cite{cha2021co2l}, \textbf{ARI}~\cite{wang2022anti}, \textbf{NISPA}~\cite{gurbuz2022nispa}, 
\textbf{DDGR}~\cite{gao2023ddgr},
\textbf{SI}~\cite{ICML2017SI}, \textbf{MAS}~\cite{ECCV2018MAS}, \textbf{OLEWC}~\cite{ICML2018Online_EWC}, 
\textbf{AOP}~\cite{guo2022adaptive}, \textbf{CRNet}~\cite{li2023CRNet},
\textbf{RPS-Net}~\cite{NIPS2019RPS-Net}, \textbf{EFT}~\cite{CVPR2021EFT}, \textbf{DER}~\cite{yan2021dynamically}, \textbf{PCL}~\cite{AAAI2021PCL}, \textbf{MORE}~\cite{kim2022multi}, \textbf{CLDNet}~\cite{li2024CLDNet}, among which the original FS-DGPM uses a multi-head incremental setting where each task has a separate classifier but task identities are not provided in our experiments; (ii) recent prompt-based methods over the pre-trained ViT: \textbf{L2P}~\cite{wang2022learning}, \textbf{DualPrompt}~\cite{wang2022dualprompt}, \textbf{CODA-Prompt}~\cite{smith2023coda}; and (iii) non-CIL baselines: stochastic gradient descent for sequential training (\textbf{SGD}; approximate lower bound) and joint multiple-task learning (\textbf{MTL}; approximate upper bound). 

\textbf{Training.} We refer to every aforementioned method's original codebases for implementation and hyper-parameter selection to ensure the best performance. We repeat each experiment five times with randomly shuffled task orderings to get the mean and the standard deviation estimates. More implementation details for \textit{architectures}, \textit{hyper-parameters}, and \textit{metrics} are provided in Appendix~\ref{Training_Details}.

\subsection{Main Comparison Results}

\begin{table*}[t]
	\caption{Comparison results on the split CIFAR-100 dataset with ResNet-18, including MORE that originally uses DeiT-S/16~\cite{touvron2021training}. Scale ratio approximately gives the \% of the final memory budget (MB) and the initial model size (MB), averaged over the two sequences. All methods allow to start with the same pre-training or learning from scratch and only the winner results are reported. 
	}
	\label{Table_CIFAR}
	\vskip 0.08in
	\begin{center}
		\begin{small}
			\begin{sc}
				\begin{tabular}{lccccc}
					\toprule
					\multicolumn{1}{l}{\textbf{Method}} &  \multicolumn{1}{c}{\textbf{Scale ratio}} & \multicolumn{2}{c}{\textbf{CIFAR-100/10}} & \multicolumn{2}{c}{\textbf{CIFAR-100/25}} \\
					&(\%) $\downarrow$ & ACA(\%) $\uparrow$ & BWT $\uparrow$ & ACA(\%) $\uparrow$ & BWT $\uparrow$ \\ \midrule 
					SGD &- & ~7.75\scriptsize{$\pm$0.17} & - & ~3.18 \scriptsize{$\pm$0.23} & - \\
					NISPA &$>$200  &37.60\scriptsize{$\pm$0.73} &$-$0.25\scriptsize{$\pm$0.02}  &29.50\scriptsize{$\pm$0.66} &$-$0.31\scriptsize{$\pm$0.01}  \\
					LOGD  &119.26   &47.45\scriptsize{$\pm$0.31} &$-$0.13\scriptsize{$\pm$0.03}  &48.71\scriptsize{$\pm$0.23} &$-$0.16\scriptsize{$\pm$0.00}  \\
					RPS-Net &$>$200 &58.95\scriptsize{$\pm$0.25} &$-$0.18\scriptsize{$\pm$0.01}  &57.43\scriptsize{$\pm$0.35} &$-$0.19\scriptsize{$\pm$0.01} \\
					DDGR &$>$200 &59.84\scriptsize{$\pm$0.57} &-  &59.15\scriptsize{$\pm$0.63} &-  \\ 
					IL2M  &119.26   &60.14\scriptsize{$\pm$0.68} &$-$0.10\scriptsize{$\pm$0.01}  &61.33\scriptsize{$\pm$0.32} &\textbf{$-$0.02}\scriptsize{$\pm$0.02} \\
					BiC &155.34  &61.03\scriptsize{$\pm$0.71} &$-$0.09\scriptsize{$\pm$0.00}  &60.24\scriptsize{$\pm$0.59} &$-$0.08\scriptsize{$\pm$0.04}  \\ 
					PCL &146.02 &63.58\scriptsize{$\pm$0.37} &$-$0.11\scriptsize{$\pm$0.01}  &62.84\scriptsize{$\pm$0.43} &$-$0.12\scriptsize{$\pm$0.01}  \\
					Co$^2$L &$>$200  &64.31\scriptsize{$\pm$0.47} &$-$0.15\scriptsize{$\pm$0.02}  &63.67\scriptsize{$\pm$0.28} &$-$0.11\scriptsize{$\pm$0.01}  \\
					DER  &$>$200   &65.29\scriptsize{$\pm$1.01} & $-$0.16\scriptsize{$\pm$0.01}  &63.54\scriptsize{$\pm$0.97} &$-$0.18\scriptsize{$\pm$0.01}  \\
					CLDNet &115.82 &65.42\scriptsize{$\pm$0.36} &-  &64.98\scriptsize{$\pm$0.42} &-  \\ 
					MORE &120.08 &67.13\scriptsize{$\pm$1.03} &-  &66.95\scriptsize{$\pm$0.98} &-  \\  
					FS-DGPM &191.71 &67.54\scriptsize{$\pm$0.36} &$-$0.23\scriptsize{$\pm$0.02}   &68.45\scriptsize{$\pm$0.38} &$-$0.27\scriptsize{$\pm$0.04}  \\ 
					Ours &\textbf{114.26} &\textbf{69.65}\scriptsize{$\pm$0.14} & \multicolumn{1}{r}{\textbf{$-$0.01}\scriptsize{$\pm$0.01}} &\textbf{70.16}\scriptsize{$\pm$0.20} & $-$0.03\scriptsize{$\pm$0.01}  \\   \bottomrule
				\end{tabular}
			\end{sc}
		\end{small}
	\end{center}
	\vskip -0.1in
\end{table*}

\begin{table}[t]
	\caption{Comparison results on the split ImageNet-R dataset using pre-trained ViT. Buffer counts the number of exemplars saved for rehearsal. Forgetting (denoted by $\mathcal{F}$) is negatively correlated with BWT. All results except ours and CODA-Prompt~\cite{smith2023coda} are extracted from~\cite{wang2022dualprompt}. Note that the original CODA-Prompt uses an easier accuracy-related AIA metric than others and we have aligned it here.}
	\label{Table_IMR}
	\vskip 0.10in
	\begin{center}
		\begin{small}
			\begin{sc}
				\begin{tabular}{lccc}
					\toprule 
					Method & Buffer $\downarrow$ & ACA(\%) $\uparrow$ & $\mathcal{F}$ $\downarrow$ \\ \midrule
					ER & \multirow{5}{*}{1 000}& 55.13\scriptsize{$\pm$1.29} & 0.35\scriptsize{$\pm$0.52} \\
					BiC &  & 52.14\scriptsize{$\pm$1.08} & 0.37\scriptsize{$\pm$1.05} \\
					GDumb &  & 38.32\scriptsize{$\pm$0.55} & - \\ 
					DER++ &  & 55.47\scriptsize{$\pm$1.31} & 0.35\scriptsize{$\pm$1.50} \\
					Co$^2$L &  & 53.45\scriptsize{$\pm$1.55} & 0.37\scriptsize{$\pm$1.81} \\ \midrule
					L2P & \multirow{4}{*}{0} & 61.57\scriptsize{$\pm$0.66} & 0.10\scriptsize{$\pm$0.20} \\
					DualPrompt &  & 68.13\scriptsize{$\pm$0.49} & 0.05\scriptsize{$\pm$0.20} \\
					CODA-Prompt &  & 69.01\scriptsize{$\pm$0.55} & 0.05\scriptsize{$\pm$0.20} \\
					Ours$^\dag$ &  & \multicolumn{1}{l}{65.45\scriptsize{$\pm$0.85}} & 0.03\scriptsize{$\pm$0.20}~  \\
					Ours$^{\dag\dag}$ &  & \multicolumn{1}{l}{\textbf{70.32}\scriptsize{$\pm$0.55}} & \textbf{0.02}\scriptsize{$\pm$0.20}~  \\ \midrule
					Upper bound &- & 79.13\scriptsize{$\pm$0.18} &- \\ 
					\bottomrule
				\end{tabular}
			\end{sc}
		\end{small}
	\end{center}
	\vskip -0.15in
\end{table}

\textbf{MNIST-10/5 and FashionMNIST-10/5.} 
Table \ref{Table_MNIST_FMNIST} extensively compares different baselines on two standard benchmark datasets adapted for CIL.  Our AutoActivator provides strong CIL performance with respect to three measurements. (i) ACA: it achieves the best average accuracy of 97.32\% and 88.46\%, surpassing the second-best competitors by 1.81\% and 2.45\%, respectively; (ii) BWT: it behaves with almost zero forgetting during sequentially learning five 2-class tasks, like some work in the task-incremental learning (TIL) scenario where task identities are required to match specific masks~\cite{kang2022forget}; and (iii) Memory budget: our final \#model is slightly ($\sim$12\%) larger than the rehearsal-based ones but our method needs no \#exemplar buffers. These results indicate that AutoActivator well trades off between model accuracy and memory budget. This makes sense for real-world applications under privacy-sensitive and resource-limited CIL scenarios. Meanwhile, the steady standard deviations show that our method has strong task-order robustness, with similar results regardless of random orderings for 5 independent runs. 

\textbf{CIFAR-100/10 and CIFAR-100/25.}
Table \ref{Table_CIFAR} reports two more challenging task sequences evenly split by the widely-used visual benchmark CIFAR-100. Based on the empirical evaluation, our method achieves the best average classification accuracy of 69.65\% and 70.16\%, improving upon the second-best method by 2.11\% on CIFAR-100/10 and 1.71\% on CIFAR-100/25, respectively. Interestingly, although pre-training implicitly alleviates the effects of catastrophic forgetting, it is not necessarily translated to CIL performance. We observe that some well-known methods still suffer from a large forgetting given the pre-trained backbone compared with a learning-from-scratch paradigm, similar findings can be observed in~\cite{wang2022dualprompt}. In addition to average classification accuracy, our method outperforms the selected state-of-the-art methods on scale ratio, indicating that the network expansion is compactly commensurate with the intrinsic complexity of a task sequence. Again, the proposed method yields competitive standard deviations indicating the superiority task-order robustness.

\begin{figure*}[ht]
	\vskip -0.05in 
	\makeatletter\def\@captype{table}\makeatother
	\begin{minipage}{0.63\textwidth}
		\setlength{\tabcolsep}{1.2mm} 
		\caption{Parameter analysis of supervisory mechanisms. We report the ACA, the cumulative number of nodes (Nodes), and the whole running time (Time) per task sequence under different $l$ and $T_{max}$.}
		\label{Table_Additional_Analysis}
		\vskip -0.1in  
		\begin{center}
			\begin{small}
				\begin{sc}
					\scalebox{0.93}{
						\begin{tabular}{llcccccc}
							\toprule
							\multirow{2}{*}{\textbf{Step size}} & 
							\multirow{2}{*}{\textbf{$T_{max}$}} & 
							\multicolumn{3}{c}{\textbf{MNIST-10/5}} & \multicolumn{3}{c}{\textbf{FashionMNIST-10/5}} \\ 
							&  & ACA(\%)$\uparrow$ & Nodes$\downarrow$ & Time(s)$\downarrow$ & ACA(\%)$\uparrow$ & Nodes$\downarrow$ & Time(s)$\downarrow$ \\ \midrule
							\multirow{4}{*}{\begin{tabular}[c]{@{}l@{}}$l=1$\end{tabular}} 
							&1   &96.86  &692  &12.87  &88.22  &702  &13.13  \\
							&10  &97.08  &640  &18.38  &88.32  &623  &16.26  \\
							&50  &97.30  &586  &28.28  &88.54  &574  &30.50  \\
							&200 &97.31  &532  &62.45  &88.53  &512  &48.99  \\ \midrule
							\multirow{4}{*}{\begin{tabular}[c]{@{}l@{}}$l=10$\\\end{tabular}} 
							&1   &96.82  &700  &6.76   &88.19  &720  &6.83  \\
							&10  &97.05  &680  &10.92  &88.37  &680  &13.39  \\
							&50  &97.24  &670  &19.74  &88.51  &660  &17.95  \\
							&200 &97.21  &650  &42.83  &88.52  &650  &33.40   \\
							\bottomrule
						\end{tabular}
					}
				\end{sc}
			\end{small}
		\end{center}
		\vskip 0.1in 
	\end{minipage}
	\hspace{0.2in} 
	\makeatletter\def\@captype{figure}\makeatother
	\begin{minipage}{0.33\textwidth}
		\setlength{\tabcolsep}{0.8mm} 
		\begin{center}
			\centerline{\includegraphics[width=0.85\columnwidth]{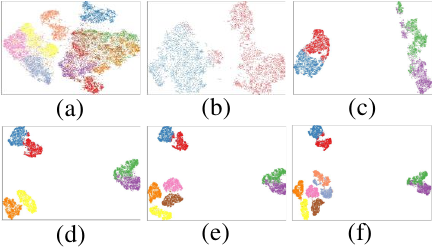}}
			\caption{t-SNE visualization where each color represents a class. (a) Mixed raw sample space based on the training data of ten classes as a reference. (b)-(f) Well-clustered representation space based on neural units' outputs after learning two classes per session.  }
			\label{Visualization}
		\end{center}
		\vskip -0.3in
		\vskip 0.1in
	\end{minipage}
	\vspace{-0.13in}
\end{figure*}

\textbf{ImageNet-R-200/10.} 
Table~\ref{Table_IMR} records the performance of compared methods starting from the same ImageNet pre-trained ViT-B/16 for the split ImageNet-R sequence. This yields two versions of our system: Ours$^\dag$ refers to only one pre-trained ViT being used; Ours$^{\dag\dag}$ refers to attaching complementary prompts to two pre-trained ViT, resembling what DualPrompt learns two sets of disjoint prompt spaces but are untouched during CIL process. We observe that it is still particularly challenging for rehearsal-based methods with a buffer size of 1 000. By contrast, there is an obvious gain for the recent emerging prompt-based methods, even though they are rehearsal-free. It is worth mentioning that Ours$^{\dag\dag}$ outperforms the strongest competitors DualPrompt and CODA-Prompt by 2.19\% and 1.31\%, respectively. This implies that prompt-based methods have not come close to exhausting their capacity for accuracy, e.g., combining with our method for a marginal boost. The proposed connectionist model, which tailors neural unit dynamics with its convergence property theoretically guaranteed, can thus serve as a general CIL classifier. Meanwhile, our method shows competitive performance on the metric Forgetting. In a nutshell, prompt-based methods exhibit a wise utilization of pre-trained transformer-based backbones while our AutoActivator seeks an effective and scalable implementation. 
\subsection{Analysis of Modifying Neural Unit Dynamics}
We now conduct an extensive empirical investigation to display the effectiveness of modules in the proposed novel connectionist model. This also covers the ablation study and parameter analysis, which pertains to each component serving for CIL. We first investigate two key components in supervisory mechanisms through the lens of CIL: (1) step size $l$ for recruiting node(s) each time---we refer to $l = 1$ as a one-by-one version and $l \geq 2$ as a batch-by-batch version; and (2) the maximum times of random generation $T_{max}$---it determines the number of attempts for mining candidate nodes that meet with a supervisory mechanism. 

\begin{table}[t]
	\caption{Influence of expected accuracy in supervisory mechanisms. We report the neural unit size per task ($t=1,2,\dots,5$) and the average number of parameters per class.}
	\label{Table_Ab1}
	\begin{center}
		\begin{small}
			\begin{sc}
				\begin{tabular}{ccccccc}
					\toprule
					$R(t)$ & 1 & 2 & 3 & 4 & 5 & \#param (M)  \\  \midrule
					99\%  &200  &200  &50  &30  &200  &0.0535    \\
					98\%  &200  &200  &20  &20  &100  &0.0425   \\
					95\%  &20   &70   &20  &10  &20   &0.0110  \\ 
					\bottomrule
				\end{tabular}
			\end{sc}
		\end{small}
	\end{center}
	\vskip -0.20in
\end{table}

\begin{table}[t]
	\caption{Effectiveness of reactivation process, measured by the ACA(\%) under different components.}
	\label{Table_Ab2}
	\begin{center}
		\begin{small}
			\begin{sc}
				\begin{tabular}{lccc}
					\toprule
					Component I  & \ding{55} & \ding{51} & \ding{51} \\
					Component II  & \ding{55} & \ding{55} & \ding{51} \\ \midrule
					MNIST-10/5 & 13.69 & 97.18 & \textbf{97.21} \\
					FashionMNIST-10/5 & 22.05 & 86.46 & \textbf{88.37} \\
					CIFAR-100/10 & 5.56 & 69.21 & \textbf{69.27} \\
					ImageNet-R & 4.98 & 63.25 & \textbf{65.83}  \\ 
					\bottomrule
				\end{tabular}
			\end{sc}
		\end{small}
	\end{center}
	\vskip -0.20in
\end{table}

\textbf{One-by-One v.s. Batch-by-Batch Version.}  It can be observed from Table~\ref{Table_Additional_Analysis} that both versions could achieve the same-level results. However, under the equal $T_{max}$ setting, the one-by-one version ($l=1$) shows superiority in expanding more compact neural units as adding only one node each time is potentially more targeted; By contrast, the batch-by-batch version ($l=10$) could significantly reduce the time requirements. It is worth mentioning that the results of both versions under proper $T_{max}$ (e.g., $T_{max} = 50$ used in our experiments) are at an acceptable level.

\textbf{The Maximum Times of Random Generation.} 
Also from Table~\ref{Table_Additional_Analysis}, we observe that a larger $T_{max}$ contributes to shrinking the Nodes while slightly improving the ACA without sacrificing the Time too much. However, an excessively high value would make the construction of the candidate pool time-consuming, while a too little one would give rise to learning instability for failure in ﬁnding nodes that are useful enough. Hence, $T_{max}$ is related to both opportunity and efficiency in the process of node generation.  

\textbf{Representation Learning of Neural Units.}
We also visualize the representation learning of scalable neural units. Fig. \ref{Visualization} depicts the t-SNE visualization~\cite{van2008visualizing} of the raw sample space and the output representations of neural units. We observe that the same classes are well clustered while different classes are properly separated. Therefore, the representation learning of neural units could provide useful information for decision-making, i.e., to adjust the decision boundary for all classes simultaneously, and facilitate the dynamic stepwise update of output weights in the least-square sense.

\textbf{The Actual Expansion Quota.} 
To indicate the expansion quota on given problems, we gather the neural unit size and count the number of parameters (\#param). Taking FashionMNIST-10/5 as an example, we specify the maximum number of nodes $L_{max}(t) = 200$ as one of the termination criteria and vary another termed expected accuracy $R(t)$ for node expansion. Table \ref{Table_Ab1} shows that the final expansion quota is commensurate with the intrinsic complexity of every task, e.g., only $30$ nodes are recruited for Task 4 under $R(t) = 99\%$. Meanwhile, the node expansion slightly introduced additional parameters, about 0.0535 M for each class, which only equals the level of replaying 7 exemplars.

\textbf{Ablation Study of Activation Thresholds.} We then validate the effectiveness of reactivation for defying the inter-class confusion in CIL. Note that component I by Eq. (\ref{threshold_t}) rectifies inter-class confusion and component II by Eq. (\ref{threshold_dis}) calibrates the results. Table~\ref{Table_Ab2} shows that our model would suffer serious forgetting without them. By contrast, using Component I could tackle this issue in most cases, together with Component II for a marginal boost. In particular, the latter works well for the case of substantial intra-class variability, demonstrating that AutoActivator could exactly reactivate the involved neural units without recourse to task identities. More insights and analysis are given in Appendix~\ref{Additional_Analysis}. 

\section{Conclusion}
We propose to harness neural unit dynamics for efficient and scalable CIL. Unlike most architecture-based methods whose expansion criterion relies on changes of the loss, which lacks theoretical guarantees, the supervisory mechanism narrows the gap using a universal approximation theorem. The reactivation paradigm pioneered is biologically plausible, and we believe this has a lot of potential room for exploration. Sufficient comparison experiments and ablation analysis display the effectiveness of our model. Other interesting investigations that we leave for future work include combining our approach with class-imbalance sequences, which may benefit from an AutoActivator-like algorithm.  


%
%
\section*{Acknowledgements}
This work was supported by the National Natural Science Foundation of China under Grants 623B2040, U1913602, and 61936004, and the 111 Project on Computational Intelligence and Intelligent Control under Grant B18024.

\section*{Impact Statement}

%
%

This paper presents work whose goal is to advance the field of General Machine Learning. We would like to discuss the potential societal consequences that may result from the limitation of our work. The supervised learning dataset we used by following most existing CIL methods involves the availability of a task boundary during training. From a practical perspective, however, it can be a restrictive assumption and somewhat not be real-world applicable. We leave further study (e.g., novel class discovery technique) on this as future work.
Besides, we assume the availability of a pre-trained backbone. While this assumption is widely accepted, as pre-trained models have become a common asset in advanced vision communities, it is important to consider that biases and fairness issues present in the original model may persist during the CIL process. To mitigate the bias and fairness issues, we strongly recommend that users conduct a thorough examination of the pre-trained models.


\bibliography{example_paper}
\bibliographystyle{icml2024}

\newpage
\appendix
\onecolumn
%
%


\section{Proof of Theorem \ref{Theorem_2}} \label{Proof_Theorem_1}
\begin{proof} 
	First, we deduce the intermediate values w.r.t. output weights $\beta_l(t)$, the counterpart of Eq. (\ref{Bete_L}) in Proposition \ref{Proposition_1}.
	\begin{equation} \label{Deduction_output_weights}
		\begin{split}
			\mathcal{L} & = \Vert { f - { f_L}} \Vert^2 \\
			& = \Vert  f - { f_{L-l}} - ( g_{L-l+1} \beta_{L-l+1}(t) + \dots +  g_L \beta_L(t)) \Vert^2\\
			& = \Vert  e_{L-l}(t) - [ g_{L-l+1}, \dots , g_L][  \beta _{L-l+1}(t), \dots , \beta_L(t)]^\mathrm{T} \Vert^2 \\
			& = \Vert  e_{L-l}(t) -  G_{l} \beta_{l}(t) \Vert^2 \\
			& = \sum_{c = 1}^{C_t} \langle  e_{L-l,c}(t) -  G_{l} \beta_{l,c}(t),  e_{L-l,c}(t) -  G_{l} \beta _{l,c}(t) \rangle \\
			& = \sum_{c = 1}^{C_t} \big( \langle  e_{L-l,c}(t),  e_{L-l,c}(t) \rangle - 2\langle  e_{L-l,c}(t),  G_{l} \beta _{l,c}(t) \rangle + \langle  G_{l} \beta_{l,c}(t),  G_{l} \beta_{l,c}(t)\rangle \big)  \\
			& = \Vert  e_{L-l}(t) \Vert^2 - \sum_{c=1}^{C_t} 2\langle  e_{L-l,c}(t),  G_{l} \beta_{l,c}(t) \rangle + \sum_{c=1}^{C_t} \langle  G_{l} \beta_{l,c}(t),  G_{l} \beta_{l,c}(t) \rangle
		\end{split}
	\end{equation}
	\noindent where $ G_l=[ g_{L-l+1},\dots, g_L]$, $ \beta_l(t)=[  \beta _{L-l+1}(t), \dots , \beta_L(t)]^\mathrm{T}$, and $ e_{L-l}(t)= f_L -  f_{L-l}$.
	
	Take the derivative of Eq. (\ref{Deduction_output_weights}) w.r.t. $ \beta_{l,c}(t)$, we have
	\begin{equation}\label{Derivative}
		\frac{\partial \mathcal{L}}{\partial  \beta_{l,c}(t)} = -2  G_l^\mathrm{T} e_{L-l,c}(t) + 2  G_l^\mathrm{T}  G_l  \beta_{l,c}(t)
	\end{equation}
	\noindent By setting Eq. (\ref{Derivative}) to zero, we have
	\begin{equation}\label{intermediate_output_weights}
		\beta_{l,c}(t) = ( G_l^\mathrm{T}  G_l)^\dag  G_l^\mathrm{T}  e_{L-l,c}(t)
	\end{equation}
	Then, denote by $ \beta_l(t) = [ \beta_{l,1}, \dots,  \beta_{l,C_t}]$ and $ e_{L}(t) = e_{L-l}^\star(t) -  G_l \beta_l(t)$ the intermediate values. Given optimal results $[{\beta}_1^\star(t), \dots, {\beta}_L^\star(t)]$ by Eq. (\ref{beta_LSM_2}), let $ e_{L}^\star(t) =  f - \sum_{j=1}^{L}\beta_j^\star(t) g_j$ ($ e_{0}^\star(t) =  f$). For the progression $\Vert  e_{L}^\star(t) \Vert^2$ , we have 
	\begin{equation}
		\begin{split}
			\Vert  e_{L}^\star(t) \Vert^2 &\leq \Vert  e_{L}(t)\Vert^2\\
			&=\langle e_{L-l}^\star(t) -  G_{l} \beta_l(t),  e_{L-l}^\star(t) -  G_{l} \beta_l(t)\rangle \\
			&\leq \Vert  e_{L-l}^\star(t)\Vert^2 - \Vert  G_{l} \beta_l(t)\Vert^2\\
			&\leq \Vert  e_{L-l}^\star(t) \Vert^2
		\end{split}
	\end{equation}
	We note that the progression $\Vert e_{L}^\star(t) \Vert^2$ is monotonically decreasing. Using Eqs. (\ref{delta_2}-\ref{Inequalities_2}), we can further obtain
	\begin{equation}
		\begin{split}
			&\Vert  e_{L}^\star(t) \Vert^2 - (r(t) + \mu_L(t)) \Vert  e_{L-l}^\star(t)\Vert^2 \\
			\leq & \Vert  e_{L}(t)\Vert^2 - (r(t) + \mu_L(t)) \Vert  e_{L-l}^\star(t)\Vert^2 \\
			= & \sum_{c=1}^{C_t}\langle e_{L-l,c}^\star(t) -  G_{l} \beta_{l,c}(t),  e_{L-l,c}^\star(t) -  G_{l} \beta_{l,c}(t)\rangle
			- \sum_{c=1}^{C_t} (r(t) + \mu_L(t)) \langle e_{L-l,c}^\star(t),  e_{L-l,c}^\star(t) \rangle \\
			= & \sum_{c=1}^{C_t}\big((1 - r(t) - \mu_L(t)) \langle e_{L-l,c}^\star(t),  e_{L-l,c}^\star(t) \rangle
			- \sum_{c=1}^{C_t}2\langle e_{L-l,c}^\star(t),   G_{l} \beta_{l,c}(t)\rangle + \sum_{c=1}^{C_t}\langle  G_{l} \beta_{l,c}(t), G_{l} \beta_{l,c}(t)\rangle \\
			= & \sum_{c=1}^{C_t}\delta_{L,c}^\star(t) \Vert  e_{L-l}^\star(t)\Vert^2 -  e_{L-l,c}^{\star\mathrm{T}}(t) G_l( G_l^\mathrm{T}  G_l)^\dag  G_l^\mathrm{T}  e_{L-l,c}^\star(t) \\
			\leq & \sum_{c=1}^{C_t} \delta_{L,c}^\star(t) - \langle e_{L-l,c}^\star(t),  G_l \beta_{l,c}(t) \rangle \\
			\leq & 0   
		\end{split}
	\end{equation}
	\noindent Therefore, we have $\Vert  e_{L}^\star(t) \Vert^2 \leq (r(t) + \mu_L(t)) \Vert  e_{L-l}^\star(t)\Vert^2$. Note that $0 < r(t) < 1$ and $\lim_{L \to +\infty }\mu_L(t) = 0$, i.e., $\lim_{L \to +\infty }\Vert  e_{L}^\star(t) \Vert = 0$. This completes the proof. 
\end{proof}

\section{Algorithm description for AutoActivator}
\label{alg_1}
To better illustrate our method, the whole procedure of training and testing is provided in Algorithm~\ref{Alg_AutoActivator}. We now show the expansion process from the perspective of matrix/vector multiplication. Note that nodes added to each scalable neural unit are in a fully-connected manner but different neural units on a certain layer are not fully connected to the next layer (see Figure~\ref{Overview}), for example, their weights are stored in a list specific to a certain layer. This is significantly different from a fully-connected (FC) layer or multi-layer perceptron (MLP) in that our classifier layer is additionally parameterized by the activation threshold for expansion-based decision-making, i.e., reactivating the required neural units in different lists. 

\begin{algorithm}[htbp]
	\caption{ActoActivator Training and Test Algorithm}
	\label{Alg_AutoActivator}
	\textbf{Input}: Datasets $\{{D}_t\}_{t=1}^T$; Termination criteria of expanding hidden unit $t$: the maximum number of nodes $L_{max}(t)$ and expected accuracy $R(t)$ 
	\begin{algorithmic}[1] 
		\STATE \textbf{\textit{\# During Sequential Training}} \\
		\FOR{$t=1, 2,\dots,T$}
		\WHILE{None of the termination criteria is satisfied}
		\STATE Recruit randomly generated nodes based on supervisory mechanisms by Eq. (\ref{xi_2})
		\STATE Update output weights by Eq. (\ref{W_L+1})
		\ENDWHILE
		\STATE Compute activation thresholds by Eq. (\ref{threshold_t}) \\
		\ENDFOR \\
		
		\STATE \textbf{\textit{\# At Inference Time}} \\
		\STATE Given batches of test instances from a certain task
		\STATE Reactivate the required hidden units by Eq. (\ref{threshold_dis}) \\
	\end{algorithmic}
	\textbf{Output}: Task identity $t$ and the predicted classes $c$
\end{algorithm}

In our algorithm, we compactly expand the network on multiple layers. Among them, the former layers are built under the guidance of supervisory mechanism (Section~\ref{Sec_Supervisory_Mechanisms}) while the final classifier layer is step-wise updated by close-formed solutions (Section~\ref{Sec_Reactivation}). Importantly, when it comes to our expansion process, the final classifier layer together with at least one (or multiple) former layer(s) is required. This is because the mentioned two different types of layers work together to complete the inference/forward pass in our algorithm.

Without the loss of generality, we take a two-layer AutoActivator on the split MNIST (five 2-classification tasks) as an example. For the first task whose input has 784 dimensions, we start with modeling the neural unit from scratch with step size $l$ nodes (e.g., $l=10$) that are randomly generated but recruited for expansion under the guidance of supervisory mechanism. When it meets the termination criteria, it yields input weight matrix $W_{in} \in R^{784 \times L_1}$ (e.g., $L_1=200$) from the former layer and output weight matrix $W_{out}  \in R^{L_1 \times 2}$ from the classifier layer. These two matrices are respectively saved in two lists corresponding to two different types of layers, which are also bound to an activation threshold (denoted by $threshold(1)$). For the second task whose input has 784 dimensions, similarly, we start with expanding the former layer by progressively recruiting additional $l$ nodes (e.g., $l=10$). When it meets the termination criteria,  it yields the newly added input weight matrix $W_{in}^\prime  \in R^{784 \times L_2}$ (e.g., $L_2=150$) for the former layer and output weight matrix $W_{out}^\prime   \in R^{L_2 \times 2}$ for the classifier layer. These two matrices are bound to an activation threshold (denoted by $threshold(2)$). The final weight of the former layer is kept in a list $[W_{in}, W_{in}^\prime]$ without row/column-wise concatenation, and the same goes for the final weight of the classifier layer that is kept in a list $[W_{out}, W_{out}^\prime]$. In this case, the final dimension of the weight in that former layer is composed of separate $784 \times L_1$ and $784 \times L_2$, instead of $(784+784) \times (L_1+L_2)$; the final dimension of the weight in that classifier layer is composed of separate $L_1 \times 2$ and $L_2 \times 2$, instead of $(L_1+L_2) \times (2+2)$. Given the input feature vector from either of trained tasks has $m=784$ dimensions, only the input-output weight matrix pair that achieves its activation threshold serves for making the decision. Thus the size of output vector is 2 dimensional with remaining deactivated. Note that since the size of $L_1$ and $L_2$ is adaptively decided by the complexity/difficulty of a given task, which are usually different, the input feature vector (of size $m$) should be the original input from that task, e.g., $m=784$ in the task sequence. This learning paradigm is dramatically different from a multi-head classifier
. Our theoretical contribution guarantees the model’s convergence property on learning sequential mappings. The above also works for multi-layer AutoActivator. Additionally, AutoActivator can be extended to convolution modules or injected into some advanced backbones such as ResNet or ViT, which accommodates the real-world scenario where pre-training is usually involved as an initial step.

\section{Additional implementation details} \label{Training_Details}
\subsection{Data splits and architectures} 
\label{DS_A}
We run experiments on extensive datasets adapted for CIL under different widely used backbones, which are implemented in PyTorch with NVIDIA RTX 3080-Ti GPUs. For fair comparisons, (i) all methods select the same or similar-sized neural network architectures; (ii) following the settings in~\cite{ke2021achieving, AAAI2021PCL, rios2022incdfm, NMI2019OWM, wu2022class}, all methods allow to start from the same pre-training unless otherwise specified; and (iii) the data for pre-training can not include in that of CIL, e.g., we manually remove the overlapping classes. The resulting data splits and architectures used in our experiments are shown in Table~\ref{Table_Datasplits}.

\begin{table}[ht]
	\caption{Details of the data splits and the selected architectures for pre-training and CIL. These are what we exactly used in our experiments of the main text unless otherwise specified.}
	\label{Table_Datasplits}
	\vskip 0.15in
	\begin{center}
		\begin{small}
			\begin{sc}
				\begin{tabular}{lccc}
					\toprule
					\multirow{2}{*}{\textbf{Dataset}} & \multirow{2}{*}{\textbf{Archirecture}} & \multicolumn{2}{c}{\textbf{Data split}} \\ \cmidrule{3-4}
					&  & Pre-training & CIL \\ \midrule
					MNIST & \multirow{2}{*}{MLP} & \multirow{2}{*}{\ding{55}} & MNIST-10/5  \\
					FashionMNIST &  &  & FashionMNIST-10/5   \\ \midrule
					
					CIFAR-100 & ResNet-18 & Tiny-ImageNet & CIFAR-100/10, CIFAR-100/25 \\ 
					ImageNet-R & ViT-B/16 &ImageNet  &ImageNet-R-200/10 \\
					\bottomrule
				\end{tabular}
			\end{sc}
		\end{small}
	\end{center}
	\vskip -0.1in
\end{table}

\subsection{Hyper-parameter} 
\label{Hyper-parameter}
We carefully reproduce the selected baselines and use the hyper-parameter settings by referring to their original source code. When conducting experiments with different datasets, we keep about 10\% of the training data from each task for validation. With regard to baselines, we use the SGD optimizer with an initial learning rate (0.001 for MNIST, FashionMNIST; 0.01 for the remaining) and do much tuning. For the rehearsal-based approaches, we keep a random exemplar set of 2k per task sequence or $\sim$200 per class by following the similar setting in~\cite{NIPS2019RPS-Net, hsu2018re}. For regularization-based approaches, the penalty coefficient is from set \{100, 1 000, 10 000, 100 000\}. For architecture-based approaches, we pay attention to their model size after learning all tasks. In our method, the maximum number of nodes $L_{max}(t)$ and expected accuracy $R(t)$ of neural unit $t$ for task $t$ are problem-dependent and not fixed. Instead, we perform the early termination criteria at the level of node expansion instead of epochs~\cite{serra2018overcoming} or phases~\cite{gurbuz2022nispa}, by tracking the lowest value of the residual error achieved so far on the validation set. This way is flexible in determining the most appropriate hyper-parameter settings without over-ﬁtting, e.g., for MNIST-10/5 and FashionMNIST-10/5, we use $L_{max}(t) = 200$ and $R(t) = 99\%$; for CIFAR-100, we use $L_{max}(t) = 1 000$ and $R(t) = 90\%$ (CIFAR-100/10), and $L_{max}(t) = 500$ and $R(t) = 80\%$ (CIFAR-100/25). Similarly, we empirically set the step size $l = 10$ for node expansion each time and the maximum times of random generation $T_{max}=50$ (see Table \ref{Table_Additional_Analysis} for more details); $r(t)=0.9$ and $\mu_L(t)=\frac{1-r(t)}{L+1}$ based on Theorem \ref{Theorem_2}.

To ensure a fair comparison, when in the absence of pre-training, every comparison method uses a similar-sized neural network architecture that is fully trainable; when in the presence of pre-training, every comparison method starts from the same pre-trained backbone as a base session. Then, following the settings in \cite{AAAI2021PCL, NMI2019OWM, hayes2020remind, ke2021achieving, wang2022dualprompt}, we keep the pre-trained model untouched for methods such as PCL and ours or still make it fully trainable for methods that could
not learn well with a frozen backbone, as evaluated in~\cite{wang2022dualprompt}. That is, all methods allow to start with the same pre-training or learning from scratch and only the winning results are reported.

\subsection{Metrics} 
\label{Metrics}
We evaluate all considered baselines based on the following metrics (higher is better): \textbf{Average Classification Accuracy (ACA)}, i.e., $\text{ACA}^T$, measures the test performance of the final model at hand on all $T$ tasks seen so far: 
\begin{equation}\label{ACA}
	\text{ACA}=\frac{1}{T}\sum_{t=1}^T R_{T,t}
\end{equation}
\noindent where $R_{T,t}$ is the test accuracy for task $t$ after training on task $T$; \textbf{Average Incremental Accuracy (AIA)} involves the intermediate results as each step/session proceeds, say 
\begin{equation}\label{AIA}
	\text{AIA} = \frac{1}{T}\sum_{t=1}^T ACA^t
\end{equation}
Among them, it is typically $ACA^1 > ACA^2 > \dots > ACA^t > \dots > ACA^T$. Therefore, ACA is more challenging than AIA, and our experiments focus more on the former while using the latter for a direct comparison with certain of competitors.
\textbf{Backward Transfer (BWT)}~\cite{NIPS2017GEM} indicates a model's ability in knowledge retention, averaged over all tasks:
\begin{equation}\label{BWT}
	\text{BWT}=\frac{1}{T-1}\sum_{t=1}^{T-1} R_{T,t}-R_{t,t}
\end{equation}
\noindent Negative BWT values mean that learning new tasks causes forgetting past tasks while a model with BWT = 0 can be considered forgetting-free~\cite{kang2022forget}. 

As an additional metric (fewer is better), we report the \textbf{Memory Budget} by aligning the memory cost of network parameters and old samples, i.e., switching them to a 32-bit floating number. In this way, both the final model size (\#model, MB) and exemplar buffers (\#exemplar, MB) are counted into the memory budget (MB), calculated with an approximate summation of them. For example, the budget for saving a simple 2-layer MLP (with [784-400-400-10] neurons) converts to 478 410 floats$\times$4 bytes/float $\div$ (1 024$\times$1 024) bytes $\approx$ 1.82 MB; the budget for saving 2k samples (with 28×28 gray-scaled pixels) converts to 1 568 000 floats$\times$ 4 bytes/float $\div$ (1 024$\times$1 024) bytes $\approx 5.98$ MB, as reported in Table~\ref{Table_MNIST_FMNIST}.

\section{Additional Comparison Results} \label{Additional_Analysis}

\subsection{Experiments on Unevenly Split Task Sequence}
Unlike the intra-sequence balanced CIL where a dataset with $C$ classes evenly divided into $T$ tasks, e.g., MNIST-10/5 and CIFAR-100/10, we now preliminarily investigate the \textit{intra-sequence imbalanced CIL} where a dataset with $C$ classes are unevenly split into $C_t$ classes---CIFAR-\{100($C_t$)$|C_t \neq C_{t+1}$\}. This yields CIFAR-\{100(10), 100(20), 100(30), \dots\} and CIFAR-\{100(2), 100(4), $\dots$\}, in which the value in parentheses indicates the number of classes for the current task. Figure~\ref{Figure_CIFAR} evaluates different algorithms under such a more realistic CIL where tasks are not evenly split. Compared with Table \ref{Table_CIFAR}, almost all the methods experience performance degradation since the intra-sequence imbalanced case has something to do with the choice of architecture or expansion. This is reflected in the different numbers of both classes and samples within each task. In the resulting 4-task and 10-task sequences, our method gets the best ACA and shows absolute superiority in BWT values. These results demonstrate that the proposed method introduces a supervisory mechanism to guide network expansion, whose growth size fully considers the intrinsic complexity of each task sequentially presented.

\begin{figure}[ht]
	\begin{center}		\centerline{\includegraphics[width=0.5\columnwidth]{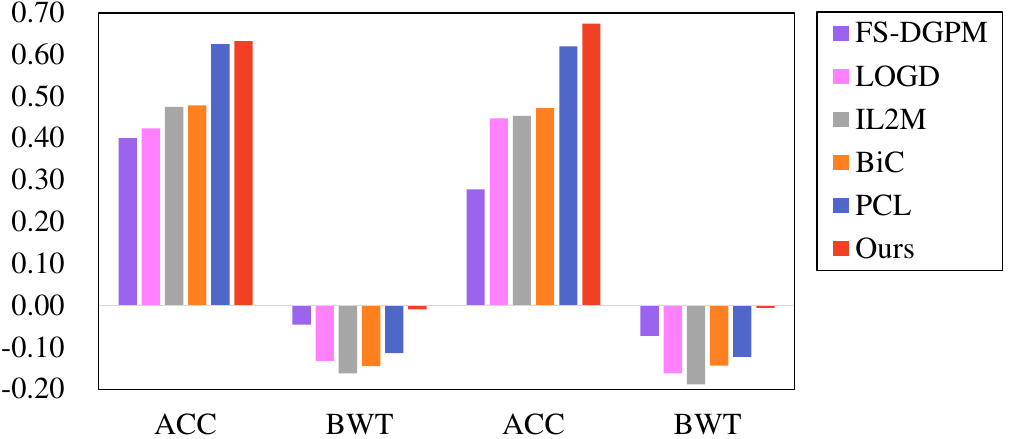}}
		\caption{Performance comparison on the intra-sequence imbalanced case. \textit{Left:} CIFAR-\{100(10), 100(20), 100(30), 100(40)\}; \textit{Right:} CIFAR-\{100(2), 100(4), 100(6), $\dots$\}.}
		\label{Figure_CIFAR}
	\end{center}
	\vskip -0.2in
\end{figure}

\subsection{Experiments on the Scalability for CIL}
To show the scalability of our method, we further display the changes of model sizes during sequential training. All methods are built upon the same backbone architecture ResNet-18 ($\sim$44.6MB). Note that RPS-Net and MORE need extra exemplar buffers. We report the actually involved model size (i.e., paths) for RPS-Net. It can be observed from Figure~\ref{model_tasks} that the proposed method outperforms the competitors. That is, after incrementally learning 10 tasks, it remains the model size relatively unchanged, without relying on exemplar buffers. We believe this is promising for practical CIL under resource-limited and privacy-sensitive scenarios.   

\begin{figure}[ht]
	\begin{center}		\centerline{\includegraphics[width=0.35\columnwidth]{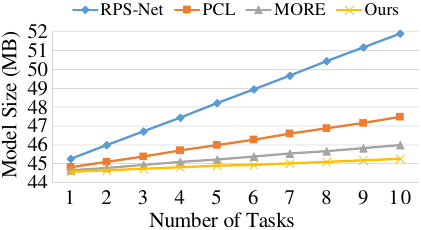}}
		\caption{Growth of the network as the number of tasks increases for CIFAR-100/10.}
		\label{model_tasks}
	\end{center}
	\vskip -0.2in
\end{figure}

\subsection{Experiments on the Training Time and Computational Costs}
As the proposed method uses a supervisory mechanism-based node recruitment step followed by a weight update step (where matrix inversion is needed), it is natural to think of the training time cost, e.g., comparing it with simpler rehearsal-based methods. Since different methods have very different requirements in computational burden, Table~\ref{Table_time} records their per-epoch running time, demonstrating our superiority in developing a fast and easy-implementation CIL model under selected parameters ($l=10$, $T_{max}=50$). As our work uses a dynamic network architecture, we also provide the computational costs measured by floating point operations per second (FLOPS). The calculation of FLOPs is affected by the size of the network and input. For a fair comparison, we follow the exact settings in AANet~\cite{liu2021adaptive}, with ResNet-32 for $32\times 32$ CIFAR-100 and ResNet-18 for $224\times 224$ ImageNet-Subset. The results yielded from ours and AANets are 70.12 M (1.82 G) and 140.00 M (3.64 G) on CIFAR-100 (ImageNet-Subset), respectively. This demonstrates the effectiveness and efficiency of our method in constructing dynamic network architecture for CIL.

\begin{table}[ht]
	\caption{Comparison results with rehearsal-based methods on running time per epoch.}
	\label{Table_time}
	\vskip 0.15in
	\begin{center}
	\begin{small}
	\begin{sc}
 	\begin{tabular}{lcccccc}
		\toprule
		Method & GEM & LOGD & RPS-Net & IL2M & FS-DGPM  & Ours  \\ \midrule
		Time (s) & 48 & 333 & 56 & 50 & 63  & 12 \\ \bottomrule 
	\end{tabular}
	\end{sc}
	\end{small}
	\end{center}
	\vskip -0.10in
\end{table}

\subsection{Experiments on Single CIL and Online CIL}
In addition to the standard CIL scenario, here we consider another two variants. One is single-class incremental learning that learns one class at a time. This is the most common case in practice because once a new class is encountered, we want to learn it immediately rather than wait for a few new classes to occur and learn them together. Following the setting in PCL~\cite{AAAI2021PCL}, Table~\ref{Table_varients} reports the comparison results for the Single CIL. In this scenario, the results except ours are drawn from that of PCL. It's worth mentioning that PCL, namely per-class continual learning, particularly excels at class-incremental learning one-by-one. For the CIFAR-100/100 (100 tasks), PCL yields 63.61\% and the proposed method is superior to it by 1.87\%. 

Another is Online CIL, in which a model learns new classes continually and data can only be observed once. In this scenario, the results except ours are taken from their original papers and the four competitors need the exemplar buffer. It is observed from Table~\ref{Table_varients} that our method outperforms ASER~\cite{shim2021online}, PRC~\cite{lin2023pcr}, OnPro~\cite{wei2023online}, GSA~\cite{guo2023dealing} on CIFAR-100/10 with fewer memory (buffer) usage. This demonstrates the effectiveness of the proposed method in the Online CIL scenario.

\begin{table}[ht]
	\caption{Comparison results on the CIFAR-100/100 for Single CIL and CIFAR-100/10 for Online CIL.}
	\label{Table_varients}
	\vskip 0.15in
	\begin{center}
	\begin{small}
	\begin{sc}
		\begin{tabular}{llll}
			\toprule
			\multicolumn{2}{c}{\textbf{Single CIL}} & \multicolumn{2}{c}{\textbf{Online CIL}}     \\
			Method          & ACA(\%) $\uparrow$          & Method & ACA(\%) $\uparrow$                       \\ \midrule
			EWC             & 2.93         & ASER   & 14.0                      \\
			RPS-Net         & 4.13         & PCR    & 25.6                      \\
			OWM             & 63.26        & OnPro  & 30.4                      \\
			PCL             & 63.61        & GSA    & 31.4                      \\
			Ours            & 65.48        & Ours   & 32.5 \\ \bottomrule
		\end{tabular}
		\end{sc}
	\end{small}
	\end{center}
	\vskip -0.10in
\end{table}

\subsection{Comparison Results When Using More Additional Parameters}

In the main text, Table~\ref{Table_IMR} adopts a ViT-B/16 transformer model pre-trained on ImageNet-21K and then incrementally learns tasks from the well-established 200-class Split Imagenet-R. In this setting, our method outperforms the selected baselines by only incurring 0.6 ($\times 10^6$) additional number of parameters (Additional No. Params) except for the pre-trained ViT. Now we further compare our method with competitors using more Additional No. Params. When more additional parameters are built upon a pre-trained model, strong baselines like SLCA~\cite{zhang2023slca} and RanPAC~\cite{mcdonnell2023ranpac} show better performance than that of all methods in Table~\ref{Table_IMR}. In Table~\ref{Table_ANP}, when our method enlarges its model size (Additional No. Params is about 3.9$\times 10^6$), denoted by Ours$^*$, its ACA is very approaching SLCA and RanPAC. Since our method belongs to the theoretical and embryonic approach, highlighting a fair comparison by considering both accuracy and memory usage, the result is still promising. Furthermore, as SLAC can be naturally plug-and-play with other CIL approaches, we make this combination with Ours$^*$, denoted by Ours$^{**}$. We can achieve extra performance improvement, and the results surpass SLCA and RanPAC by 1.91\% and 1.01\%, surprisingly getting closer to the upper bound.

\begin{table}[t]
	\caption{Comparison results when more additional parameters are built upon a pre-trained model.}
	\label{Table_ANP}
	\vskip 0.15in
	\begin{center}
	\begin{small}
	\begin{sc}
		\begin{tabular}{lcc}
			\toprule
			Method & ACA(\%) $\uparrow$ & Additional No. Params ($\times 10^6$) $\downarrow$ \\ \midrule
			SLAC & 77.00 & $\sim$123.7 \\
			RanPAC & 77.90 & $\sim$12.5 \\
			Ours$^*$ & 76.32 & $\sim$3.90 \\
			Ours$^{**}$ & 78.91 & $\sim$127.5 \\
			Upper bound & 79.13 & - \\ \bottomrule
		\end{tabular}
		\end{sc}
	\end{small}
	\end{center}
	\vskip -0.10in
\end{table}

\subsection{Further Comparison on ImageNet-100 and ImageNet-1K Using the Metric AIA}
Before concluding our empirical evaluation, Table~\ref{Table_AIA} provides a comparison between the proposed method and some top-performing methods on ImageNet-100 and ImageNet-1K, measured by AIA. We first test our methods on ImageNet-100. Following the benchmark protocol used in PODNet~\cite{douillard2020podnet}, AANets~\cite{liu2021adaptive}, DER\cite{yan2021dynamically}, FOSTER~\cite{wang2022foster}, and ACIL~\cite{zhuang2022acil}, we start from a model trained on 50 base classes (B50), and the remaining 50 classes are divided into splits of 10 steps. 
Note that we directly take the reported results on ImageNet-100 B50 from their original papers to report the best performance. It can be observed that our method is superior to the strongest baseline by 0.52\%. Similarly, we test our methods on ImageNet-1K. For our method, we build it upon the available backbone provided by ACIL where a ResNet-18 was well-trained based on half of the ImageNet-1K datasets. This ImageNet-1K B500 setting is also used in PODNet, AANets, and ACIL but not in DER and FOSTER. For this, we mark DER and FOSTER with B0, i.e., training on ImageNet-1K from scratch. Although training from scratch seems more attractive, they rely on storing 20 000 old task exemplars that our method does not. It can be observed that our method outperforms the strongest baseline by 0.82\%.

\begin{table}[ht]
	\caption{Comparison results using the metric AIA(\%). $^*$ denotes additional exemplar buffer is required during CIL.}
	\label{Table_AIA}
	\vskip 0.15in
	\begin{center}
	\begin{small}
	\begin{sc}
	\begin{tabular}{lcc}
		\toprule
		Method & ImageNet-100 & ImageNet-1K \\ \midrule
		PODNet$^*$ & 74.33 & 64.13 \\
		AANets$^*$ & 75.58 & 64.85 \\
		DER$^*$ & 77.73 & 66.73 (B0) \\
		FOSTER$^*$ & 77.54 & 68.34 (B0) \\
		ACIL & 74.76 & 64.84 \\
		Ours & 78.25 & 69.16 \\ \bottomrule
	\end{tabular}
	\end{sc}
	\end{small}
	\end{center}
	\vskip -0.10in
\end{table}

\end{document}